\documentclass[letterpaper]{article} % DO NOT CHANGE THIS
\usepackage[preprint]{aaai2027}
\usepackage[hyphens]{url}  % DO NOT CHANGE THIS
\usepackage{graphicx} % DO NOT CHANGE THIS
\usepackage{natbib}  % DO NOT CHANGE THIS AND DO NOT ADD ANY OPTIONS TO IT
\usepackage{caption} % DO NOT CHANGE THIS AND DO NOT ADD ANY OPTIONS TO IT
\usepackage{algorithm}
\usepackage{algorithmic}
\usepackage{placeins}
\usepackage{booktabs}
\usepackage{amsmath}
\usepackage{amssymb}
\usepackage{amsthm}
\usepackage{tikz}
\usetikzlibrary{arrows.meta,positioning}

\newcommand{\Fcand}{F_{C}}
\newcommand{\Fref}{F_{R}}
\newcommand{\Phiadm}{\Phi_{\mathrm{adm}}}
\newcommand{\true}{\textsc{true}}
\newcommand{\false}{\textsc{false}}
\newcommand{\unk}{\textsc{unknown}}
\newcommand{\ur}{\textsc{unknown\_resource}}
\newcommand{\absent}{\textsc{absent}}
\newcommand{\uu}{\textsc{unknown\_unsupported}}

\newcommand{\fwithin}{\textsc{false\_within\_declared\_universe}}

\newtheorem{definition}{Definition}[section]
\newtheorem{proposition}[definition]{Proposition}

\newtheorem{corollary}[definition]{Corollary}

\newcommand{\system}{\textsc{ModelEquivBench}}

\title{ModelEquivBench: Certifying Multi-Relational Evaluation of \\LLM-Generated Optimization Models}
\author{Penglin Zhu, Jungang Xu\corresponding}
\affiliations{
School of Computer Science and Technology, University of the Chinese Academy of Sciences\\
zhupenglin21@mails.ucas.ac.cn, xujg@ucas.ac.cn
}

\begin{document}
\maketitle

\begin{abstract}
Large language models increasingly generate optimization models from natural
language, but existing evaluation often reduces a generated model and its
ground truth to a single \emph{equivalent}/\emph{not-equivalent} verdict or an
execution-success rate---labels that are neither independently checkable nor
faithful to the multiple distinct senses in which two formulations can agree.
We present \system{}, a certifying, multi-relational evaluation system that
reports a per-pair \emph{semantic profile} $E0$--$E6$: model construction and
exact ingestion ($E0$), verified representation alignment ($E1$), same-space
and projected feasible-set relations ($E2$, $E3$), objective-order equivalence
($E4$), optimal-value equality ($E5$), and optimizer-set equivalence ($E6$).
Each decided entry carries relation-appropriate, independently re-checkable
evidence: replayable traces or explicit maps for $E0$--$E1$, exact-rational
certificates for positive $E2$--$E6$ conclusions, and explicit witnesses for
supported negatives. Incomplete mapping search, unsupported structure, and
resource limits produce typed \textsc{unknown} or N/A outcomes rather than
guesses, while unmet prerequisites are reported as \textsc{absent}. Using
\system{} to evaluate three model snapshots---GPT-5.4, Claude Sonnet 4.6, and
Qwen3.5-397B-A17B---on the same frozen cohort of 173 base problems (346 cells
per model) under a no-repair protocol, the resulting profiles expose
distinctions that coarse baselines do not represent: $49$, $35$, and $25$
cells contain executable candidates that are nevertheless certified negative
on at least one supported relation, and $25$, $8$, and $18$ structural
rejections occur on pairs for which $E2$ certifies mapped feasible-set equality
under a verified map. The three model snapshots fail at different stages of
the profile and therefore cannot be meaningfully reduced to a single accuracy
score.
\end{abstract}
% --- begin inlined sections.tex ---
\section{Introduction}

Large language models (LLMs) are increasingly used to turn natural-language
problem descriptions into runnable optimization models
\citep{ramamonjison2023nl4opt,ahmaditeshnizi2024optimus,xiao2024chain,huang2025orlm}.
Assessing whether a generated model is \emph{correct}, however, remains
unsettled. Common signals include (i) \emph{execution success}---the code runs
and a solver returns a number---and (ii) a single
\emph{equivalent}/\emph{not-equivalent} verdict against a ground-truth model,
produced by value comparison or structural graph matching
\citep{wang2025orgeval,zhai2025equivamap}. Neither by itself answers all of the
semantic questions relevant to formulation correctness. Execution
success says nothing about whether the model \emph{means} the right thing: a
program can build, export, and solve a model that encodes the wrong feasible
region or objective. A single global equivalence label, in turn, conflates
several genuinely different questions---do the feasible sets coincide, do the
objectives induce the same ordering, are the optimal values equal, are the
optimizer sets in bijection?---and is typically returned without a proof that a
third party could independently re-check.

We take a different stance. Given an LLM-generated (\emph{candidate}) model and
a \emph{reference} model, we ask: \emph{in which distinct semantic senses do
they agree or disagree, and which of those conclusions can be independently
certified?} Our answer is \system{}, a certifying, multi-relational evaluation system
that reports a per-pair \emph{semantic profile} of seven dimensions,
$E0$--$E6$, summarized in Table~\ref{tab:defs}. The dimensions range from model
construction and exact ingestion ($E0$), through verified representation
alignment ($E1$) and feasible-set relations in the aligned space ($E2$) and
under an affine lift ($E3$), to objective order ($E4$), optimal value ($E5$),
and optimizer sets ($E6$). Crucially, $E0$--$E6$ form a \emph{profile}, not a
ladder: the indices order the dimensions' \emph{definedness} prerequisites,
not their logical strength, and equal feasible sets do not imply equal
objective order, value, or optimizers. There is no ``deepest passing level''
and no scalar collapse.

\system{} is \emph{certifying}: each decided entry carries
relation-appropriate evidence that an independent checker can re-verify. $E0$
uses a replayable execution/ingestion trace, $E1$ an explicit admissible map,
and positive $E2$--$E6$ conclusions exact-rational certificates (Farkas,
affine-lift, objective-identity, or primal--dual evidence). Supported negatives
carry a failing trace or explicit witness. Incomplete search, unsupported
structure, and resource limits yield typed \unk{} or N/A outcomes; failed
prerequisites make later dimensions \absent{}. The bidirectional
feasible-set containment engine we call \emph{Certifying Mapped-Containment}
(CMC) is the principal technical engine for $E2$ and the supported part of
$E3$; it is a component inside \system{}, not a replacement for the
$E0$--$E6$ profile.

We instantiate \system{} on the supported envelope of linear and
bounded-discrete models and evaluate three model snapshots---GPT-5.4, Claude
Sonnet 4.6, and Qwen3.5-397B-A17B---on a frozen cohort of $173$ base problems,
each under paired Structured and Unstructured conditions ($346$ cells per
model), with one generation per condition at temperature $0.0$ and \emph{no
repair or resampling}. Their profiles differ sharply. GPT-5.4 produces
 ingestible candidates for $334/346$ cells, compared with $156/346$ for Claude
Sonnet 4.6 and $196/346$ for Qwen3.5-397B-A17B. Conditional $E1$ coverage is
nevertheless similar: $277/334$ ($82.9\%$), $130/156$ ($83.3\%$), and
$164/196$ ($83.7\%$). Qwen3.5-397B-A17B additionally records $17$ provider/API
errors before candidate creation; these are reported as E0 \absent{}, not
model-quality failures. Exact certification reveals $49$, $35$, and $25$
execution-success overestimations and $25$, $8$, and $18$ structural
rejections despite $E2$-certified mapped feasible-set equality. Our contributions are: (1) \system{},
the $E0$--$E6$ certifying multi-relational evaluator and its typed-abstention
discipline; (2) a certifying implementation with independent re-verification
of every decided fact; and (3) a three-model formal study showing distinctions
that execution-only, value-only, and structural baselines do not represent.
The contribution is an evaluator and certification system, not a new benchmark
dataset: Bench4Opt supplies the experimental instances. Supplement Sections
A--E give the full definitions, certificate soundness proofs, boundary cases,
and explicit non-implication examples.
% --- begin inlined figures/pipeline_float.tex ---
\begin{figure}[t]
\centering
\includegraphics[width=\columnwidth]{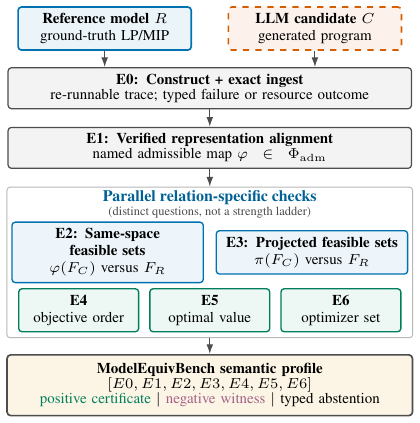}
\caption{The \system{} certifying $E0$--$E6$ workflow. $E0$ and $E1$ establish
prerequisites; $E2$--$E6$ ask parallel, relation-specific questions. Decided
facts carry replayable certificates or witnesses, while unresolved entries
remain typed outcomes rather than a global verdict.}
\label{fig:pipeline}
\end{figure}

% --- end inlined figures/pipeline_float.tex ---
\section{Related Work}

\subsection{LLM Optimization Modeling and Evaluation}

\paragraph{LLM optimization modeling.} A growing line of work prompts or
fine-tunes LLMs to formulate optimization models from text, including the
NL4Opt competition \citep{ramamonjison2023nl4opt}, agentic decomposition
systems \citep{ahmaditeshnizi2024optimus,xiao2024chain}, and trained modelers
\citep{huang2025orlm,lu2025optmath}. These works focus on \emph{generation};
we focus on \emph{certifying evaluation} of generated models against a
reference.

\paragraph{Evaluating generated mathematical programs.} Evaluation has largely
relied on execution success and optimal-value matching, or on a single
equivalence verdict. EquivaMap uses an LLM to propose mappings between decision-variable
spaces and then verifies feasibility and optimality preservation
\citep{zhai2025equivamap}; EquiBench studies LLMs' program-equivalence
reasoning \citep{wei2025equibench}. These approaches target an overall
equivalence judgment, whereas \system{} reports separately certified relations and
typed abstentions.

\subsection{Structural Comparison and Certifying Verification}

\paragraph{Structural model comparison.} ORGEval compares optimization models
by graph-theoretic canonicalization \citep{wang2025orgeval}. Structural
matching is efficient but can reject pairs whose feasible sets coincide after
a verified change of representation; our $E2$ relation certifies exactly this
same-space, mapped feasible-set equality without claiming full-profile
equivalence.

\paragraph{Proof certificates and exact verification.} Certifying algorithms
attach independently checkable evidence to their answers, from proof-carrying
code \citep{necula1997proof} to certified integer-programming reasoning
\citep{cheung2017verifying,hoen2024certifying,doornmalen2023proof,bogaerts2022certified}.
Feasible-set containment rests on Farkas' lemma and polyhedral theory
\citep{schrijver1986theory}; projected polyhedra and extended formulations are
studied by \citet{yannakakis1991expressing}, \citet{conforti2013extended},
\citet{kellner2015containment}, and \citet{liberti2009reformulations}. Exact
rational LP/MIP methods \citep{applegate2007exact,cook2013hybrid} make the
resulting evidence checkable without additional floating-point error during
verification. We adapt these tools into \system{}, in which every decided fact is
independently re-verified using exact arithmetic.

\section{The ModelEquivBench Framework}

\subsection{Models, Evidence Policy, and Outcome States}

\paragraph{Models.} Let $M$ have $n_M$ decision variables collected in
$x=(x_1,\ldots,x_{n_M})^\top$, with per-coordinate domains
$D_{M,i}$ (continuous, integer, or binary) and
$D_M=\prod_i D_{M,i}\subseteq\mathbb{R}^{n_M}$. Its feasible set is
$F_M=\{x\in D_M:A_{\mathrm{ub}}x\le b_{\mathrm{ub}},\,
A_{\mathrm{eq}}x=b_{\mathrm{eq}}\}$; it has linear objective
$f_M(x)=\mathrm{obj}_M^\top x$, sense in $\{\min,\max\}$, optimal value
$\mathrm{opt}(M)$ in the extended reals, and optimizer set
$\mathrm{argopt}(M)$. Empty-set conventions apply, and all coefficients are
stored as exact rationals. We write $R,C$ for reference and candidate, with
ingested feasible sets $\Fref,\Fcand$ (bounds folded into inequalities).
\system{} derives all claims from the models, never from declared labels.

\paragraph{Global disciplines.} Four rules govern every dimension:
(i) positive facts carry replayable traces, verified maps, or exact-rational
certificates as appropriate; (ii) supported negatives carry failing traces or
re-checkable witnesses; (iii) every fact from $E2$ onward is tied to a named,
verified map $\varphi$; and (iv) the evaluator is fail-closed---missing
prerequisites, unsupported classes, extraction failures, and resource limits
produce typed abstentions, never guessed \false{} outcomes. Unmet prerequisites
are recorded separately as \absent{}.
% --- begin inlined tables/e0e6_defs.tex ---
% Table 1: E0-E6 dimensions, evidence types, and abstention conditions.
\begin{table*}[t]
\centering
\small
\setlength{\tabcolsep}{4pt}
\renewcommand{\arraystretch}{1.15}
\begin{tabular}{@{}p{0.055\textwidth}p{0.30\textwidth}p{0.30\textwidth}p{0.27\textwidth}@{}}
\toprule
\textbf{Dim.} & \textbf{Relation compared} & \textbf{Positive evidence / negative witness} & \textbf{Typed abstention} \\
\midrule
E0 & Candidate code builds a structurally valid exact-rational model that ingests & Re-runnable parse/build trace; FALSE carries the failing trace & \textsc{unknown\_resource} (exec.\ cap); \textsc{unknown\_unsupported} (e.g.\ quadratic) \\
E1 & A verified admissible alignment $\varphi\in\Phi_{\mathrm{adm}}$ exists (type/domain-compatible) & Explicit $\varphi$ (permutation, complement, sign, or affine lift), re-checked exactly & \textsc{unknown}: no map found in the searched family (coverage-limited, never \textsc{false}) \\
E2 & Same-space feasible-set relation $\varphi(F_C)$ vs.\ $F_R$ (equal / strict rel.\ / strict restr.\ / incomparable) & Two-way Farkas containment multipliers (equal); separating witness (\textsc{false\_within}) & \textsc{unknown} if no certified-equal and universe not exhaustive; \textsc{unknown\_resource} \\
E3 & Projected relation of a lifted candidate: $\pi(F_C)$ vs.\ $F_R$ under section $\sigma$ & Farkas obligations for $\pi(F_C)\!\subseteq\! F_R$, $\sigma(F_R)\!\subseteq\! F_C$, $\pi\!\circ\!\sigma\!=\!\mathrm{id}$ & \textsc{unknown} outside the affine-lift schema; \textsc{n/a} if no auxiliary vars \\
E4 & Objective-order equivalence under $\varphi$ & Identity/lift/complement objective identity; FALSE: two oppositely ranked points & \textsc{unknown}; \textsc{n/a} \\
E5 & Optimal-value equality $\mathrm{opt}(R)=\mathrm{opt}(C)$ & Two LP weak-duality sandwiches (LPs only; not used for integer models); implications from E2/E3 & \textsc{unknown}; \textsc{unknown\_resource}; \textsc{n/a} \\
E6 & Optimizer-set bijection $\mathrm{argopt}(R)\!\leftrightarrow\!\mathrm{argopt}(C)$ & Identity or 0--1 complement involution certificate (positive-only) & \textsc{unknown} (no certified-negative type); \textsc{n/a} \\
\bottomrule
\end{tabular}
\caption{The E0--E6 semantic profile. Each dimension is a distinct relation with its own certificate/witness type and typed abstention vocabulary. The dimensions are prerequisites for one another's \emph{definedness}, not rungs of a monotone equivalence ladder; a pair receives a profile, not a single label. N/A means a relation is mathematically inapplicable, whereas \textsc{absent} means that an upstream prerequisite was not met. \textsc{ur}=\textsc{unknown\_resource}; \textsc{false\_within}=\textsc{false\_within\_declared\_universe}.}
\label{tab:defs}
\end{table*}

% --- end inlined tables/e0e6_defs.tex ---
\paragraph{Outcome states and definedness.} A profile entry is not merely
\true{} or \false{}. A dimension may also be \unk{} because the supported
procedure established neither a positive conclusion nor a supported negative
conclusion, \ur{} because its resource budget
was exhausted, \textsc{unknown\_unsupported} because the model lies outside the
implemented envelope, or N/A because the relation is mathematically
inapplicable. We reserve \absent{} for a different situation: an upstream
prerequisite was not established, so the later dimension was never evaluated.
For example, if no candidate model ingests, $E4$ is \absent{}, not N/A; if a
same-space candidate has no auxiliary variables, $E3$ is N/A. This distinction
matters both logically and statistically. A rate such as ``$E2$ decided given
$E1$'' conditions on the dimensions that were actually entered, whereas a
cohort-wide count retains \absent{} cells in the accounting without treating
them as failures of the relation. The resulting profile is therefore a partial,
typed set of claims rather than a seven-bit vector.

\subsection{Construction and Representation Alignment}

\paragraph{$E0$: construction and exact ingestion.} $E0$ asks whether the
candidate code executes and yields a structurally valid exact-rational model
that ingests. It is \true{} with a re-runnable build trace, \false{} with the
failing parse/build/execution trace, \ur{} when the execution cap is exceeded,
and \textsc{unknown\_unsupported} for a valid but out-of-envelope construct
(e.g.\ a quadratic objective). $E0$ does not require a solver result:
infeasible and unbounded models may still be $E0$ \true{}. $E0$ is established
from ingestion and \emph{checkpointed before} any later work, so a downstream
timeout cannot corrupt it.

\paragraph{$E1$: verified representation alignment.} Because a candidate may
name and order variables differently, $E1$ searches for an admissible map from
candidate to reference variables. The admissible family is
$\Phiadm=\Phi_{\mathrm{same}}\cup\Phi_{\mathrm{lift}}$, where
$\Phi_{\mathrm{same}}$ contains type-compatible coordinate permutations
composed with per-coordinate transforms---identity, binary complement
$x\mapsto 1-x$ (0--1 only), and sign negation $x\mapsto -x$ (free continuous
only)---and $\Phi_{\mathrm{lift}}$ contains affine projection/section pairs
with injective sections, relating a lower-dimensional reference to a
higher-dimensional candidate. Candidate maps are proposed by a deterministic
grammar (exact and original name matching, type-compatible permutation,
signature matching, sign, complement, and affine lift) and each is checked for
admissibility (dimension match, permutation bijectivity, per-transform legality
by variable type, section injectivity, and the exact right-inverse identity
$\pi\circ\sigma=\mathrm{id}$ for affine lifts). $E1$ is \true{} if at least one
admissible $\varphi$ is verified and \unk{} if none is found in the searched
family. Critically, $E1$ is \emph{never} \false{}: absence of a found map is
not proof that no map exists, so the outcome is coverage-limited, not a
refutation.

\subsection{Feasible-Set Relations}

\paragraph{$E2$: same-space feasible-set relation.} Fixing a verified same-space
$\varphi$, $E2$ compares the mapped candidate feasible set $\varphi(\Fcand)$
with $\Fref$ \emph{as subsets of the same semantic variable space} and
classifies the pair as \emph{equal} ($\Fref=\varphi(\Fcand)$), \emph{strict
relaxation} ($\Fref\subsetneq\varphi(\Fcand)$), \emph{strict restriction}
($\varphi(\Fcand)\subsetneq\Fref$), \emph{incomparable}, or \emph{unknown}.
Each containment direction $P\subseteq Q$ is certified by exact nonnegative
Farkas multipliers $(\mu,\lambda)$ with $a^\top=\mu^\top G+\lambda^\top H$ and
$\beta\ge\mu^\top h+\lambda^\top h_{\mathrm{eq}}$ for every row $(a,\beta)$ of
$Q$; a failure to contain is witnessed by a feasible point of $P$ violating a
specific row of $Q$. This bidirectional, certificate-or-witness containment is
the CMC engine; multipliers and witnesses are round-tripped through an
independent verifier. Per-map relations are aggregated across the declared map
universe: $E2$ is \true{} if \emph{some} admissible $\varphi$ is certified
equal (and re-verified); it is
\textsc{false\_within\_declared\_universe} only if that universe is
\emph{exhaustive} and \emph{every} map in it separates with a re-verified
witness (no equal, no unknown); otherwise it is \unk{}. The universe is
declared exhaustive only when the type-compatible permutation family is fully
enumerable ($n\le 7$) and no binary coordinate is present, since binary
complements are not exhaustively enumerated. Thus
\textsc{false\_within\_declared\_universe} is a completeness-relative
negative---``no admissible representation in the complete declared universe
makes the feasible sets equal''---not a failed search.

\paragraph{Per-map and aggregated claims.} The distinction between a relation
under one named map and the aggregate $E2$ outcome is essential. A certified
strict restriction under a particular $\varphi$ does not rule out a second map
that makes the sets equal. Hence a positive aggregate equality is existential,
whereas a negative aggregate statement requires a complete finite universe and
a re-verified separation for every member. When the universe is incomplete,
the strongest sound conclusion after testing many non-equal maps is still
\unk{}. This asymmetric policy deliberately sacrifices recall to prevent a
failed representation search from being mislabeled as semantic inequality.
Certificates and witnesses retain the map identifier, the exact transformed
constraints, and the verified obligation, so every aggregate claim can be
decomposed into independently checkable per-map facts.

\paragraph{$E3$: projected feasible-set relation.} When the candidate carries
auxiliary variables, $E3$ compares $\Fref$ with the \emph{projection} of the
higher-dimensional $\Fcand$. With an affine projection $\pi:\mathbb{R}^{n_C}\to
\mathbb{R}^{n_R}$ and an affine section $\sigma(x)=Cx+d$, the checker certifies
three obligations---$\pi(\Fcand)\subseteq\Fref$, $\sigma(\Fref)\subseteq\Fcand$,
and $\pi\circ\sigma=\mathrm{id}$---each via verifier-accepted Farkas
multipliers or exact affine identities. Together these certify
$\pi(\Fcand)=\Fref$: the reference set is exactly the affine projection of the
candidate's feasible set. Projections outside the supported affine-lift schema
(integer-auxiliary elimination, non-affine projection, reference equalities)
yield \unk{}, and same-space cases with no auxiliary variables are N/A. If a prerequisite such
as $E0$ or $E1$ is missing, $E3$ is \absent{} rather than N/A. We stress that $E3$ is \emph{not} a level above $E2$: $E2$ compares aligned sets in
one space, while $E3$ bridges spaces of different dimension; they are different
relations, reported independently.

\subsection{Objective and Solution Relations}

\paragraph{$E4$--$E6$: objective and optimizer relations.} Under a verified
candidate-to-reference map $\varphi$ ($\varphi=\pi$ for lifts), $E4$ checks
whether $f_C$ and $f_R\circ\varphi$ induce the same ordering on candidate
feasible points with consistent optimization sense. Positives use an exact
identity in the identity/lift/complement schema (additive constants cancel in
comparisons); negatives use two oppositely ranked feasible points. $E5$ certifies \emph{optimal-value
equality} $\mathrm{opt}(R)=\mathrm{opt}(C)$ by two LP weak-duality sandwiches of
equal value or by implications from $E2$/$E3$ equality with a preserved
objective; duality certificates are \emph{refused for integer variables}, since
a relaxation bound does not certify the integer optimum. $E6$ certifies an
\emph{optimizer-set bijection} $\mathrm{argopt}(R)\leftrightarrow
\mathrm{argopt}(C)$ for a map with a certified inverse on optimizers (identity
under $E2$-equality with identical objective, or a 0--1 complement involution).
$E6$ is \emph{positive-only} in the current envelope: no certified-negative
type exists yet, so absence of a correspondence is \unk{}, never \false{}; and
$E3$-equality does not yield $E6$ because the affine section is only a right
inverse. These non-implications ($E2\not\Rightarrow E4$, $E4\not\Rightarrow
E5$, $E5\not\Rightarrow E6$, and their converses) are exactly why the framework
reports a profile rather than a level.

\paragraph{Why the relations cannot be collapsed.} The profile separates
logically independent phenomena. Two formulations can have the same feasible
set but opposite objective senses, making $E2$ true and $E4$ false. Distinct
feasible sets can share one optimum value, so $E5$ does not imply $E2$.
Likewise, equal optimal values do not identify the same optimizers, and an
optimizer correspondence says nothing about non-optimal feasible points. Even
when one relation entails another under additional premises, the evaluator
records the premises and the resulting certificate rather than silently
propagating an unqualified global label. This is also why a baseline that
checks only execution, value, or structure cannot be treated as an oracle for
the whole profile.

\section{Certifying Implementation}

\subsection{Exact Ingestion and Certificate Generation}

\system{} orchestrates a verification core;
synthesis is never trusted, and every artifact is re-verified before a decided
fact is reported. Candidate code is extracted deterministically and executed in
isolation under a $30$\,s cap; it must export an LP or MPS file, which is
ingested into an exact-rational internal model. Every serialized decimal token
is converted directly to a rational, so certification is exact with respect to
the LP/MPS coefficients and introduces no additional floating-point error. $E1$ map proposals come from
the deterministic grammar above. Feasible-set containment uses an exact-rational
LP oracle to search for Farkas multipliers (positive direction) or a separating
feasible point (negative direction); bounded 0--1 systems are handled by exact
enumeration. Affine-lift and primal--dual optimality certificates are
synthesized similarly. Every synthesized object is then re-checked by an
independent verifier---two-way containment for $E2$ equality, witness
verification for separations, the affine-lift obligations for $E3$, and the
weak-duality re-check for $E5$---and only re-verified facts are reported;
all emitted certificate and witness artifacts were independently re-verified
in all three runs (Table~\ref{tab:results}); the reported denominator counts
artifacts rather than unique cells. Evaluation is \emph{dimension-level checkpointed}, so
a per-dimension timeout yields a typed \ur{} for exactly that dimension while
preserving earlier results. The protocol enforces \emph{no repair and no
resampling}: a malformed or code-less response is a genuine $E0$ outcome, not an
occasion to re-prompt.

Gurobi plays a strictly bounded role: it builds and exports the generated
models and reads LP/MPS artifacts, and it backs the value-matching and
structural (ORGEval-style) baselines used only for comparison. It does
\emph{not} certify any $E1$--$E6$ conclusion; all framework certificates use
the exact-rational oracle and the independent verifier. Handling of transport timeouts and character-decoding failures in the harness
provides engineering robustness and is not part of the evaluation semantics.

\subsection{Independent Verification and Replay}

\paragraph{Independent checking boundary.} Certificate construction and
certificate acceptance are separate code paths. The constructor may use an
exact LP oracle to search for multipliers or witnesses, but the verifier only
receives the serialized model, map, and proposed evidence and recomputes the
required rational identities and inequalities. It does not trust solver status,
LLM text, cached verdicts, or floating-point tolerances. A failed check prevents the proposed fact from being reported as decided
rather than downgrading it to an uncertified Boolean. This separation is the operational basis for our use of
``certified'': a third party can replay the compact evidence without replaying
the generation process or accepting the constructor's control flow.

\paragraph{System interface and replay artifacts.} \system{} takes a reference
model and one generated candidate as its unit of evaluation and emits three
linked products: the typed $E0$--$E6$ profile, the evidence objects supporting
each decided entry, and a provenance manifest recording the map identifier,
resource outcome, and source artifact hashes. This interface separates model
generation from semantic assessment: a new generator can be evaluated without
changing the certifier, while an updated certifier can replay persisted
candidates without another paid model call. Dimension-level checkpoints also
make partial profiles stable under interruption---a timeout in $E5$, for
example, does not erase a previously verified $E2$ certificate. The resulting
artifact is therefore more informative than a final boolean: it states what was
proved, under which representation, what remained unresolved, and which compact
objects a third party must check to reproduce the claim. Supplement Section F
provides construction/replay pseudocode and the evidence schema.

\section{Experiments}

\subsection{Experimental Setup}\label{sec:experimental-setup}

\paragraph{Cohort.} Bench4Opt contains $394$ source records
\citep{wang2025orgeval}; in our paired representation these form $197$ base
problems. A
ground-truth--only support audit---depending solely on properties of the
reference artifacts, never on any LLM output---partitions them into a
development pilot set ($20$), a formal primary cohort ($173$), an unsupported
audit set ($4$, e.g.\ nonlinear structure outside the envelope), and a
malformed-or-missing set ($0$). We report results \emph{only} on the frozen
$173$-problem formal cohort (SHA-256 prefix \texttt{e9ad2018}). This
reference-only partition was frozen before formal generation, so no LLM output
or evaluator result influenced inclusion. The $20$-problem pilot was used only
for development and is excluded from formal results. Each base problem is presented under two paired natural-language
conditions---\emph{Structured} (legacy internal identifier \texttt{full}) and
\emph{Unstructured} (\texttt{concise})---giving $173\times 2=346$ cells per
model. The two conditions are paired by base problem. Besides condition-specific
rates, we derive base-level discordant counts, exact McNemar tests, and paired
bootstrap intervals as descriptive analyses of this frozen cohort; they are not
population-level claims.

\paragraph{Models and protocol.} We evaluate three snapshots served through
the AutoDL OpenAI-compatible endpoint: \texttt{gpt-5.4},
\texttt{claude-sonnet-4-6}, and \texttt{Qwen3.5-397B-A17B}. For each cell we take \emph{one} generation at temperature
$0.0$ with \emph{no repair} and \emph{no resampling}; the candidate is executed
under a $30$\,s cap and the $E1$--$E6$ evaluator under a uniform $120$\,s cap.
Prompts, cohort, mapping family, certificate rules, and baseline semantics are
identical across all three models. Every run uses the same frozen prompt
contract requesting one self-contained program that exports
\texttt{candidate.lp}. Provider/API errors before candidate creation are kept
separate from genuine E0 failure.

\paragraph{Baselines and reporting.} Alongside the \system{} profile, we run two comparison baselines: a
solver-based \emph{value-matching} implementation and the ORGEval-style
structural implementation used in our harness. We do not claim that the latter
is the unmodified official ORGEval implementation. Every rate names its denominator, and
\unk{}, \ur{}, \textsc{unknown\_unsupported}, N/A, \absent{},
provider/API failure, and certified negatives are kept distinct throughout;
\unk{} never means ``incorrect,'' N/A means the relation is mathematically
inapplicable, and \absent{} means a prerequisite was not met. Supplement
Sections G--J provide the complete cohort/protocol record, full outcome tables,
paired statistics, failure taxonomy, and replay instructions.

\paragraph{Outcome accounting.} We report both cohort-wide counts and
conditional coverage. Cohort-wide rows retain all $346$ cells and expose where
upstream generation or ingestion prevents later evaluation. Conditional rows
ask, among cells for which a dimension is defined, how often the evaluator
decides it: $E1$ coverage is conditioned on $E0$-\true{}, and the main $E2$
decided rate on $E1$-\true{}. The paired analysis instead uses all $173$ base
problems per model and defines binary end-to-end indicators for E0 success, E1
map-found, and E2/E5/E6 decided. Typed unknowns, N/A, and absence remain visible
in the source matrix and are never relabeled as certified \false{}.

Table~\ref{tab:results} reports the formal profiles, while
Table~\ref{tab:diag} summarizes relation-specific baseline discrepancies.
% --- begin inlined tables/diagnostics.tex ---
% Table 2: three-model baseline diagnostics.
\begin{center}
\begin{minipage}{\columnwidth}
\centering
\scriptsize
\setlength{\tabcolsep}{3.0pt}
\renewcommand{\arraystretch}{1.06}
\begin{tabular}{@{}lrrr@{}}
\toprule
Diagnostic & GPT-5.4 & Sonnet 4.6 & Qwen3.5 \\
\midrule
Execution-success overestimation & 49 & 35 & 25 \\
ORGEval rejection despite E2 equality & 25 & 8 & 18 \\
Strict value-match false acceptance & 0 & 0 & 0 \\
\bottomrule
\end{tabular}
\captionof{table}{Relation-specific baseline diagnostics. An execution-success overestimation is an E0-\textsc{true} cell with a certified negative on at least one supported relation. The second row counts structural rejections on pairs for which $E2$ certifies mapped feasible-set equality; it is an $E2$-specific disagreement, not necessarily a full-profile false rejection. The ORGEval implementation is the one used in our harness.}
\label{tab:diag}
\end{minipage}
\end{center}

% --- end inlined tables/diagnostics.tex ---% --- begin inlined tables/results_main.tex ---
% Table 3: three-model formal result summary.
\begin{table*}[t]
\centering
\scriptsize
\setlength{\tabcolsep}{3.2pt}
\renewcommand{\arraystretch}{1.08}
\begin{tabular}{@{}p{0.39\textwidth}ccc@{}}
\toprule
\textbf{Dimension / metric (denominator)} & \textbf{GPT-5.4} & \textbf{Sonnet 4.6} & \textbf{Qwen3.5} \\
\midrule
\multicolumn{4}{@{}l}{\emph{173 base problems $\times$ \{Structured, Unstructured\} $=346$ cells per model; identical SHA-256 cohort-hash prefix \texttt{e9ad2018}}}\\
\addlinespace[2pt]
E0 \textsc{true} (/346) & 334 (96.5\%) & 156 (45.1\%) & 196 (56.6\%) \\
\quad genuine E0 \textsc{false} / unsupported / provider error & 12 / 0 / 0 & 190 / 0 / 0 & 131 / 2 / 17 \\
E1 map found (/E0-\textsc{true}) & 277/334 (82.9\%) & 130/156 (83.3\%) & 164/196 (83.7\%) \\
\quad E1 \textsc{unknown} & 57 & 26 & 32 \\
E2 decided (/E1-\textsc{true}) & 114/277 (41.2\%) & 45/130 (34.6\%) & 87/164 (53.0\%) \\
\quad E2 \textsc{true} / \textsc{false\_within} & 108 / 6 & 41 / 4 & 84 / 3 \\
\quad E2 \textsc{unknown} / \textsc{unknown\_resource} & 118 / 45 & 61 / 24 & 50 / 27 \\
E3 decided & 0 & 0 & 0 \\
E4 decided & 139 & 40 & 91 \\
E5 \textsc{true} / \textsc{false} & 107 / 42 & 36 / 31 & 79 / 22 \\
E6 \textsc{true} (positive-only) & 70 & 16 & 58 \\
\addlinespace[2pt]
Certificate/witness artifacts re-verified & 358/358 & 123/123 & 250/250 \\
120\,s cap-hit cells (union over dimensions) & 70 & 35 & 39 \\
\bottomrule
\end{tabular}
\caption{Formal \system{} profiles under the same single-generation, temperature-$0.0$, no-repair protocol. Qwen3.5-397B-A17B's $17$ provider/API errors occurred before candidate creation and are reported separately from genuine E0 failure. \textsc{false\_within} abbreviates \textsc{false\_within\_declared\_universe}. Re-verification denominators count emitted dimension-verdict artifacts, not unique cells.}
\label{tab:results}
\end{table*}

% --- end inlined tables/results_main.tex ---
\subsection{Stage-Wise Model Profiles}

\paragraph{Execution does not settle semantics.} GPT, Sonnet, and Qwen produce
ingestible candidates for $334/346$, $156/346$, and $196/346$ cells. Yet their E2 decided counts
are only $114/277$, $45/130$, and $87/164$ among E1-\true{} cells. Across the
three runs, $49$, $35$, and $25$ E0-\true{} cells carry a certified negative
on at least one supported relation. An execution-only evaluator would count
every one of these as successful.

\paragraph{The evaluated snapshots fail at different stages.} GPT has high E0
coverage ($96.5\%$), so most losses occur downstream. Sonnet has the lowest E0
coverage ($45.1\%$) under the frozen output contract. Qwen lies between them:
$196$ cells are E0-\true{}, while among $329$ successful API responses $131$
are genuine E0 failures and $2$ are unsupported; another $17$ cells end in a
provider/API error before E0 is defined. Conditional E1 coverage is almost
identical across the three models ($82.9\%$, $83.3\%$, and $83.7\%$), showing
why generation, ingestion, and semantic coverage must be reported separately.

\paragraph{Conditional coverage changes the comparison.} Among cells with a
verified map, Qwen reaches an E2 decision on $87/164$ ($53.0\%$), compared with
$114/277$ ($41.2\%$) for GPT and $45/130$ ($34.6\%$) for Sonnet. This does
not define a global winner: the denominators exclude different upstream losses.
Rather, it shows how a model with lower end-to-end ingestion can have higher
conditional certifier coverage, a distinction hidden by one aggregate score.

\subsection{Baseline Discrepancies and Certified Case}

\paragraph{Exact certification exposes baseline blind spots.} ORGEval returns
\texttt{not\_equivalent} on $25$, $8$, and $18$ cells for which E2 certifies
mapped feasible-set equality for GPT, Sonnet, and Qwen, respectively. These are relation-specific discrepancies:
the structural baseline rejects the pair globally, whereas $E2$ certifies
equality of the mapped feasible sets. This comparison does not by itself assert
objective-order, optimal-value, or optimizer-set equivalence. No batch exhibits a strict value-match false acceptance, but
value matching addresses only one scalar relation and supplies no evidence for
feasible-set, objective-order, or optimizer-set claims.
% --- begin inlined figures/certified_case_float.tex ---
\begin{figure}[t]
\centering
\includegraphics[width=\columnwidth]{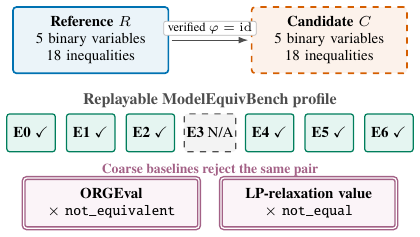}
\caption{A replayable \system{} all-applicable-positive case
(\texttt{B4O\_BASE\_0031}, Unstructured, Claude Sonnet 4.6). A verified
identity map and exact enumeration certify E2; E4--E6 independently re-verify
as true and E3 is N/A. Both coarse baselines nevertheless reject the pair.}
\label{fig:certified-case}
\end{figure}

% --- end inlined figures/certified_case_float.tex ---
\paragraph{A replayable all-applicable-positive case.}
Figure~\ref{fig:certified-case} shows a Sonnet case,
\texttt{B4O\_BASE\_0031} (Unstructured). An identity map aligns two models,
each with five binary variables and 18 inequalities; exact enumeration certifies
both $E2$ directions, and $E4$--$E6$ independently re-verify \true{} ($E3$ is
N/A). ORGEval returns \texttt{not\_equivalent}, while the LP-relaxation value
baseline returns \texttt{not\_equal}; the latter does not contradict $E5$,
which concerns the original binary optima. The paired Structured cell has the
same applicable profile, but ORGEval returns \texttt{equivalent}, illustrating
representation sensitivity. Persisted paths and SHA-256 hashes make the case
replayable.

\subsection{Coverage, Abstention, and Reporting}

\paragraph{Coverage and abstention are results, not errors.} The three runs
contain $57$, $26$, and $32$ E1-\unk{} cells; their E2 resource-timeout counts
are $45$, $24$, and $27$. At least one evaluator dimension hits the $120$\,s
cap in $70$, $35$, and $39$ cells. No formal cell receives a decided $E3$ outcome, so the
study empirically exercises E0--E2 and E4--E6 while E3 records a zero-coverage
boundary of the current affine-lift schema. The complement of a decided rate
therefore mixes search incompleteness, unsupported structure, and resource
abstention rather than forming a model error rate.

\paragraph{Auxiliary paired analysis.} For GPT and Sonnet, base-level Structured--Unstructured comparisons show descriptive
positive differences for downstream E1/E2 coverage, but no exact McNemar test
survives Holm correction. We retain this as an auxiliary observation rather
than a three-model ranking; all primary comparisons use the common aggregate
denominators in Tables~\ref{tab:diag} and~\ref{tab:results}.

\paragraph{Soundness and coverage are separate axes.} Every decided relation is
backed by evidence that can be replayed independently, while coverage reports
how often the current map family, supported schema, and resource budget reach a
decision. High coverage without evidence risks confident but ungrounded labels;
sound certification with limited coverage is transparent but incomplete. The
profile reports both, preserving certified negatives, unresolved searches,
resource limits, N/A, and upstream absence rather than forcing them into one
score.

\paragraph{Implications for evaluator design.} \system{} suggests that evaluation
reports should separate at least three quantities that are often conflated:
end-to-end generation success, conditional semantic coverage after ingestion,
and the distribution of certified positive, certified negative, and abstaining
outcomes within each relation. The same generator can look strong on the first
quantity and weak on the second, or vice versa, as the three snapshots illustrate.
A scalar ``equivalence accuracy'' cannot reveal whether errors arise from output
contract violations, representation search, semantic disagreement, unsupported
structure, or exhausted resources. Profile-level reporting also makes baseline
comparisons relation-specific: value matching may be informative for $E5$ but
silent about $E2$, while structural matching can disagree with a certified
same-space feasible-set relation. In practical benchmark use, the profile can
therefore serve both as a scorecard and as a debugging record, directing model
or prompt improvements to the stage where evidence actually fails rather than
to an undifferentiated final label.

A minimum auditable report should therefore pair the eligible cohort with
end-to-end counts, per-relation conditional coverage, outcome distributions,
re-verification rates, and the resource and map-completeness assumptions that
govern \unk{} and \textsc{false\_within\_declared\_universe}.

\paragraph{Scope of claims.} Results are limited to \texttt{gpt-5.4},
\texttt{claude-sonnet-4-6}, and \texttt{Qwen3.5-397B-A17B} on the frozen $173$-base cohort, one
generation per condition at temperature $0.0$, no repair or resampling, and
the $120$\,s evaluator policy. We claim neither universal superiority,
state-of-the-art generation, nor completeness over all mathematical programs.

\section{Limitations and Conclusion}

\paragraph{Limitations.} We evaluate three model snapshots with one generation
per condition; broader claims require more models and repeated sampling. The
supported envelope covers linear and bounded-discrete structure, while
quadratic and general nonlinear models yield
\textsc{unknown\_unsupported}. The finite $E1$ grammar is incomplete, $E3$
supports only the declared affine-lift schema, and $E6$ has no
certified-negative type. Exact rational verification also incurs resource
limits: under the $120$\,s policy, at least one dimension returns \ur{} in
$70$, $35$, and $39$ GPT, Sonnet, and Qwen cells. Results may depend on
provider-specific serving, and one sample per cell cannot quantify generation
variance. The paired prompt analysis is descriptive---no McNemar comparison
survives Holm correction---and the frozen cohort is not a random population
sample. Finally, no formal cell has a decided $E3$ outcome, so empirical
coverage of projected equivalence remains unestablished.

\paragraph{Conclusion.} \system{} reports each candidate--reference pair as a
certified $E0$--$E6$ semantic profile rather than one opaque label. The frozen
three-model study shows that execution is weak evidence of semantic
correctness, different snapshots fail at different stages, and exact evidence
with typed abstention makes both decisions and current coverage limits
auditable.

% --- end inlined sections.tex ---
\clearpage
\bibliography{references}

\clearpage
\appendix
\begin{center}
{\Large\bfseries Supplementary Material}\par
\vspace{0.5em}
{\large ModelEquivBench: Certifying Multi-Relational Evaluation of\par
LLM-Generated Optimization Models}
\end{center}
\vspace{0.5em}
\noindent
This technical appendix provides formal statements and proofs for the
E0--E6 relations, certificate and witness schemas, algorithms for profile
construction and independent verification, detailed cohort and protocol
information, complete three-model outcome accounting, paired analyses, failure
taxonomies, and reproduction instructions.
\par\medskip

\section{Formal Setup, Scope, and Typed Outcomes}
\label{supp:formal-setup}

\subsection{Optimization models and exact semantics}

\begin{definition}[Optimization model]\label{def:optimization-model}
A model $M$ is a tuple
\[
M=(n,D,A_{\rm ub},b_{\rm ub},A_{\rm eq},b_{\rm eq},c,s),
\]
where $n$ is the number of variables; $D=D_1\times\cdots\times D_n$ with
$D_i\in\{\mathbb{R},\mathbb{Z},\{0,1\}\}$ (possibly intersected with exact
rational bounds); $A_{\rm ub}x\le b_{\rm ub}$ and
$A_{\rm eq}x=b_{\rm eq}$ are exact-rational constraints; $c\in\mathbb{Q}^n$ is
a linear objective; and $s\in\{\min,\max\}$ is the objective sense. Its
feasible set and optimizer set are
\begin{align*}
F_M=\{x\in D:\;&A_{\rm ub}x\le b_{\rm ub},\\
                  &A_{\rm eq}x=b_{\rm eq}\},\\
X_M^*&=\operatorname*{argopt}_{x\in F_M} c^\top x.
\end{align*}
The optimal value $v_M^*$ is interpreted in the extended reals. For
minimization, $v_M^*=+\infty$ when $F_M=\varnothing$ and $v_M^*=-\infty$ when
the objective is unbounded below; the signs are reversed for maximization.
\end{definition}

The reference model is denoted $R$ and the generated candidate $C$. Every
numeric token serialized in LP/MPS is converted directly to a rational. Thus,
all certificate checks are exact with respect to the serialized model; the
claim does not reconstruct real-valued quantities that may have been rounded
before serialization.

\subsection{Profiles are partial typed claim sets}

The output for one candidate/reference pair is
\[
\mathcal{P}(C,R)=\bigl(E0,E1,E2,E3,E4,E5,E6\bigr),
\]
but this tuple is not a seven-bit vector and the indices are not a strength
ranking. Each component records a relation, its applicability, its evidence,
and a typed reason if no decision is returned.

\begin{table}[t]
\centering
\scriptsize
\setlength{\tabcolsep}{3pt}
\begin{tabular}{@{}p{0.39\columnwidth}p{0.52\columnwidth}@{}}
\toprule
State & Meaning \\
\midrule
\true & The stated positive relation is supported by evidence that passes the independent exact verifier. \\
\false{} / relation label & A supported negative relation has an explicit re-checkable witness. For E2 the aggregate negative is \fwithin{}, not an unrestricted universal claim. \\
\unk & The relation is applicable but the supported search did not establish either side. \\
\ur & The relation is applicable but a declared resource cap was reached. \\
\uu & A valid ingested object uses structure outside the implemented certification envelope. \\
N/A & The relation is mathematically inapplicable (e.g., E3 for a same-space candidate with no auxiliaries). \\
\absent & An upstream prerequisite was not established, so the relation was never entered. \\
\bottomrule
\end{tabular}
\caption{Typed outcome vocabulary. N/A and \absent{} are deliberately distinct.}
\label{tab:supp-outcomes}
\end{table}

\subsection{Evidence discipline and prerequisite DAG}

The framework obeys four global rules.
\begin{enumerate}
\item Every positive semantic fact is accepted only after exact independent
re-verification of a certificate.
\item Every supported negative fact carries a concrete witness.
\item E2--E6 are conditioned on a named verified representation map produced by
E1.
\item Missing prerequisites, unsupported structure, failed evidence
construction, and resource exhaustion never become guessed negatives.
\end{enumerate}

E0 gates ingestion. E1 gates map-conditioned relations. E2 and E3 ask different
feasible-set questions and do not form consecutive rungs. E4--E6 may be
considered once an admissible map exists, but none is silently inferred from a
lower index without recording the additional premises and evidence.

\section{E0 and E1: Ingestion and Representation Alignment}
\label{supp:e0-e1}

\subsection{E0: candidate construction and exact ingestion}

\begin{definition}[E0]
E0 asks whether a model response yields a deterministic candidate program
that executes within the declared cap, exports an LP/MPS model, and is ingested
into a structurally valid exact-rational model. E0 is \true{} with an
independently replayable trace; it is a genuine \false{} when a successful
model response contains no usable program, fails deterministically, or exports
an invalid model. A provider/API failure before a candidate exists leaves E0
\absent{}. A valid but unsupported model is \uu{}, and an execution-cap event
is \ur{}.
\end{definition}

This definition separates three questions that are frequently conflated:
provider delivery, candidate validity, and semantic correctness. E0 \true{}
asserts only the second. It makes no claim about E1--E6.

\begin{proposition}[E0 trace replay]
If the independent replay harness re-executes the recorded candidate program
under the specified environment and obtains the same normalized LP/MPS-derived
exact model, then the E0 positive claim is independently reproducible.
\end{proposition}
\begin{proof}
The trace specifies the candidate program, execution command, resource cap,
exported model, and deterministic exact-ingestion procedure. Replaying these
objects reconstructs the same normalized tuple in Definition~\ref{def:optimization-model}; the replay
decision does not rely on a solver's semantic judgment.
\end{proof}

\subsection{E1: admissible mappings}

A mapping aligns semantic decision variables before any relation is tested. The
general admissible schema contains:
\begin{enumerate}
\item type-compatible coordinate permutations, including identity;
\item binary complement $x_i\mapsto 1-x_i$ for 0--1 variables;
\item sign negation $x_i\mapsto -x_i$ for free continuous variables;
\item a declared affine lift/section pair for the supported E3 schema.
\end{enumerate}
The general mapping schema admits sign negation for free continuous variables.
In the reported formal run, however, the instantiated proposer manifest is the
finite family listed in Supplement Section~\ref{supp:e1-search}: stable-name
match, original-name match, identity, signature permutation, binary complement,
and affine lift. No standalone sign-negation proposal is instantiated, so E1
and aggregate E2 coverage do not quantify over such maps.

\begin{definition}[Same-space admissible map]
A same-space map is an affine bijection
\[
\varphi(x)=Px+q,
\]
where $P$ is a signed permutation matrix and $q$ is an exact offset. A negative
entry denotes either sign negation on a free continuous coordinate with zero
offset, or binary complement on a 0--1 coordinate with unit offset. It is
admissible only if dimensions, variable types, bounds, and every coordinate
transform are legal and the inverse is exact. In the reported finite proposal
family, negative entries are instantiated only by binary complements.
\end{definition}

\begin{definition}[Lift admissibility]
For $n_C>n_R$, a supported lift consists of a projection
$\pi:\mathbb{R}^{n_C}\to\mathbb{R}^{n_R}$ and an affine section
$\sigma(x)=Cx+d$ satisfying $\pi\circ\sigma=\mathrm{id}$ exactly, together with
compatible domains for all mapped variables.
\end{definition}

\begin{proposition}[Why E1 has no aggregate false]
Failure to find an admissible map in a finite searched family does not prove
that no semantics-preserving map exists outside that family. Therefore the
sound aggregate outcome is \unk{}, not \false{}.
\end{proposition}
\begin{proof}
The deterministic mapping search evaluates a declared subset of all possible
bijections and affine encodings. A negative conclusion over that subset cannot quantify over
unsearched maps. Since E2--E6 may become positive under an unsearched map, the
only sound result is a coverage-limited abstention.
\end{proof}

Each accepted E1 artifact records the mapping type, exact parameters, domain
compatibility checks, and a stable identifier used by every downstream
certificate and witness.

\section{E2: Same-Space Mapped Feasible-Set Relations}
\label{supp:e2}

\subsection{Directional containment certificates}

Let
\[
\begin{aligned}
P &= \{x:Gx\le h,\ Hx=h_{\rm eq}\},\\
Q &= \{x:Ax\le b\}.
\end{aligned}
\]
with exact rational data after domain handling. The certificate form follows
Farkas' lemma and standard polyhedral containment arguments
\citep{schrijver1986theory}. Equalities may be carried
explicitly or converted to paired inequalities where supported.

\begin{proposition}[Farkas row certificate]
\label{prop:farkas-row}
For a row $a_i^\top x\le b_i$ of $Q$, suppose there exist
$\mu_i\ge0$ and unrestricted $\lambda_i$ such that
\[
a_i=G^\top\mu_i+H^\top\lambda_i,
\qquad
b_i\ge h^\top\mu_i+h_{\rm eq}^\top\lambda_i.
\]
Then every $x\in P$ satisfies $a_i^\top x\le b_i$.
\end{proposition}
\begin{proof}
For $x\in P$,
\[
a_i^\top x=\mu_i^\top Gx+\lambda_i^\top Hx
\le \mu_i^\top h+\lambda_i^\top h_{\rm eq}\le b_i,
\]
where the first inequality uses $\mu_i\ge0$ and $Gx\le h$.
\end{proof}

\begin{corollary}[Containment]
If Proposition~\ref{prop:farkas-row} has a verified multiplier pair for every row of
$Q$, then $P\subseteq Q$.
\end{corollary}

\begin{proposition}[Separating witness]
If a rational point $z$ is independently verified to satisfy all constraints
and domains of $P$ and violates at least one row of $Q$, then $P\nsubseteq Q$.
\end{proposition}
\begin{proof}
The point $z$ belongs to $P$ but not $Q$, which directly refutes containment.
\end{proof}

For bounded 0--1 systems, exhaustive enumeration of all declared assignments is
an exact finite certificate. The replay verifier re-checks every assignment or
a compact deterministic enumeration trace.

\subsection{Per-map relation classification}

Fix one verified same-space map $\varphi$ and write
$P=\varphi(\Fcand)$ and $Q=\Fref$. Two directional checks yield:
\begin{itemize}
\item \emph{equal}: $P\subseteq Q$ and $Q\subseteq P$;
\item \emph{strict relaxation}: $Q\subseteq P$ and a witness proves
$P\nsubseteq Q$;
\item \emph{strict restriction}: $P\subseteq Q$ and a witness proves
$Q\nsubseteq P$;
\item \emph{incomparable}: witnesses refute both directions;
\item \emph{unknown}: at least one required direction is unresolved.
\end{itemize}

\begin{proposition}[Soundness of the five-way per-map diagnostic]
Whenever all evidence required by one of the first four labels above passes the
independent verifier, the corresponding set relation holds under the recorded
map $\varphi$.
\end{proposition}
\begin{proof}
Equality follows from antisymmetry of set inclusion. Each strict label combines
one certified inclusion with a verified point in the set difference. Two
verified set-difference points establish incomparability.
\end{proof}

\subsection{Aggregation across maps}

Let $\mathcal{M}\subseteq\Phiadm$ be the declared verified map universe.
\begin{definition}[Aggregate E2]
The aggregate E2 result is:
\begin{itemize}
\item \true{} if some $\varphi\in\mathcal{M}$ has certified equality;
\item \fwithin{} only when $\mathcal{M}$ is declared complete and every member
has a verified non-equality witness;
\item \unk{} otherwise.
\end{itemize}
\end{definition}

\begin{proposition}[Soundness of aggregate E2]
A \true{} aggregate means that a named admissible representation makes the
feasible sets equal. A \fwithin{} aggregate means no map in the complete
declared universe makes them equal. Neither statement quantifies over maps
outside the declared grammar.
\end{proposition}
\begin{proof}
The positive case is existential and inherits the per-map equality proof. In
the negative case, completeness of the finite declared universe plus a verified
non-equality witness for every member proves the universal claim restricted to
that universe. Without completeness, such a universal conclusion is invalid,
so the framework returns \unk{}.
\end{proof}

The implementation declares the simple type-compatible permutation universe
complete only when it is exhaustively enumerable (currently $n\le7$) and no
binary coordinate is present. Binary-complement subsets are not exhaustively
enumerated, so any binary coordinate makes the declared universe incomplete.
This conservative boundary prevents failed map search from being mislabeled as
semantic inequality.

\subsection{Degenerate feasible sets}

The empty set obeys ordinary containment conventions: $\varnothing\subseteq Q$
for every $Q$, and $P\subseteq\varnothing$ only when $P=\varnothing$. A
positive empty-set equality therefore requires evidence that both sides are
empty. A solver status alone is insufficient; the accepted artifact must be
replayable under the exact supported procedure. Unboundedness is not an E2
special case because E2 concerns sets, not objective values.

\section{E3: Projected Feasible-Set Equality Under an Affine Lift}
\label{supp:e3}

E3 applies when the candidate has auxiliary variables and a supported affine
projection/section pair has passed E1.

\begin{proposition}[Affine-lift equality]\label{prop:affine-lift}
Let $\pi:\mathbb{R}^{n_C}\to\mathbb{R}^{n_R}$ and
$\sigma:\mathbb{R}^{n_R}\to\mathbb{R}^{n_C}$ satisfy:
\begin{align}
\pi(\Fcand)&\subseteq\Fref,\label{eq:e3-a}\\
\sigma(\Fref)&\subseteq\Fcand,\label{eq:e3-b}\\
\pi\circ\sigma&=\mathrm{id}_{\mathbb{R}^{n_R}}.\label{eq:e3-c}
\end{align}
Then $\pi(\Fcand)=\Fref$.
\end{proposition}
\begin{proof}
Equation~\eqref{eq:e3-a} gives one containment. For any $x\in\Fref$,
Equation~\eqref{eq:e3-b} gives $\sigma(x)\in\Fcand$, and
Equation~\eqref{eq:e3-c} yields
$x=\pi(\sigma(x))\in\pi(\Fcand)$. Hence
$\Fref\subseteq\pi(\Fcand)$, proving equality.
\end{proof}

The two set-containment obligations are discharged using exact Farkas evidence
or bounded enumeration, while the right-inverse identity is checked by exact
matrix arithmetic. This is a deliberately restricted, certifying instance of
projected-polyhedron and extended-formulation reasoning
\citep{yannakakis1991expressing,conforti2013extended,kellner2015containment,liberti2009reformulations}. Notice that $\pi$ need not be injective; the section
$\sigma$ is injective as a consequence of $\pi\circ\sigma=\mathrm{id}$.

E3 returns N/A for same-space candidates with no auxiliary variables. It
returns \unk{} outside the declared affine-lift schema, including unsupported
integer-auxiliary elimination, non-affine projection, and reference-equality
cases outside the implemented schema. If E0 or E1 is missing, E3 is \absent{},
not N/A. No formal-cohort cell satisfies the implemented E3 certificate schema;
this is a zero-coverage boundary,
not evidence against Proposition~\ref{prop:affine-lift}.

\section{E4--E6: Objective Order, Optimal Value, and Optimizers}
\label{supp:e4-e6}

\subsection{E4: objective-order equivalence}

For a verified map pairing candidate point $y$ with reference point
$x=\varphi(y)$, define the strict and weak preference relations induced by the
recorded objective sense.

\begin{definition}[E4]
E4 is positive when, for every paired feasible $y_1,y_2$,
\[
f_R(\varphi(y_1))\preceq_R f_R(\varphi(y_2))
\iff
f_C(y_1)\preceq_C f_C(y_2),
\]
where $\preceq_M$ uses the sense of model $M$.
\end{definition}

\begin{proposition}[Positive-affine objective certificate]
If the senses agree and, on paired feasible points,
\[
f_C(y)=\alpha f_R(\varphi(y))+\beta
\qquad\text{for exact }\alpha>0,\ \beta\in\mathbb{Q},
\]
then E4 holds.
\end{proposition}
\begin{proof}
A strictly positive affine transform preserves all weak, strict, and equality
comparisons. The additive constant cancels in pairwise differences.
\end{proof}

The implemented positive forms include identity, supported lifts, and binary
complements that induce an exact positive-affine identity. A supported negative
is a pair of feasible points whose rankings disagree, re-checked against both
models and the named map.

\subsection{E5: optimal-value equality}

\begin{definition}[E5]
E5 asks whether $v_R^*=v_C^*$ as extended-real values under the original
variable domains and recorded senses.
\end{definition}

\begin{proposition}[Primal--dual optimum certificate]
For a continuous rational LP, a primal-feasible point and a dual-feasible point
with equal objective values certify the exact optimum by weak duality. If such
certificates for $R$ and $C$ yield the same rational value, E5 is positive.
\end{proposition}
\begin{proof}
Weak duality places every primal objective on one side of every dual objective.
Equality of a feasible primal and dual pair closes the bound and proves
optimality. Repeating this for both models and comparing the exact values proves
E5.
\end{proof}

For integer or binary variables, an LP-relaxation dual certificate does not
certify the integer optimum and is refused. Bounded 0--1 cases may instead use
exact enumeration. E5 can also be implied by a certified feasible-set equality
and objective identity, but the report records those premises rather than
silently transferring a global equivalence label.

\subsection{E6: optimizer-set equivalence}

\begin{definition}[E6]
E6 is positive when the recorded admissible map restricts to a bijection
between $X_C^*$ and $X_R^*$.
\end{definition}

\begin{proposition}[Optimizer transport]
Suppose $\varphi$ is a bijection between $\Fcand$ and $\Fref$ and E4 holds.
Then $y\in X_C^*$ if and only if $\varphi(y)\in X_R^*$, so $\varphi$ restricts
to an optimizer-set bijection.
\end{proposition}
\begin{proof}
If $y$ is candidate-optimal, no feasible $y'$ is strictly preferred. E4
transports this ordering to the reference set, and feasible-set bijectivity
ensures every reference competitor has a candidate preimage. Therefore
$\varphi(y)$ is reference-optimal. The reverse direction follows from
$\varphi^{-1}$.
\end{proof}

E6 accepts positive certificates only for forms with an exact certified
inverse on optimizers (identity and the supported 0--1 complement involution).
The implemented schema has no certified-negative E6 type. Consequently,
unsupported or unproved cases are \unk{}, never \false{}. E3 equality alone
does not imply E6 because its section is only a right inverse and need not
represent every candidate optimizer uniquely.

\subsection{Explicit non-implications}

\begin{table}[t]
\centering
\small
\setlength{\tabcolsep}{3pt}
\begin{tabular}{@{}p{0.25\columnwidth}p{0.66\columnwidth}@{}}
\toprule
Non-implication & Counterexample sketch \\
\midrule
E2 $\not\Rightarrow$ E4 & Same feasible set $[0,1]$; one model minimizes $x$, the other maximizes $x$. \\
E4 $\not\Rightarrow$ E5 & Same feasible set and $f_C=2f_R+1$: order is preserved but raw optimum values differ. \\
E5 $\not\Rightarrow$ E6 & On $\{0,1\}$, minimize $x$ versus minimize $1-x$: both optimum values are $0$, but the optimizers differ under identity. \\
E6 $\not\Rightarrow$ E2 & Reference feasible set $[0,1]$ and candidate feasible set $\{0\}$, both minimizing $x$, share optimizer $0$ but not the full feasible set. \\
E3 is not above E2 & E2 is a same-space relation; E3 is a projection relation across dimensions. Either may be inapplicable while the other is meaningful. \\
\bottomrule
\end{tabular}
\caption{Why the profile cannot be collapsed into a deepest passing level.}
\label{tab:nonimplications}
\end{table}

\section{Certificate Construction, Independent Replay, and Algorithms}
\label{supp:implementation}

\subsection{End-to-end evaluator}

\begin{algorithm}[t]
\caption{Prerequisite-aware construction of one E0--E6 profile}
\label{alg:profile}
\begin{algorithmic}[1]
\REQUIRE cell $z$, fixed provider configuration $\theta$, cohort hash $h$
\STATE Load the response identified by fingerprint $(z,\theta)$, or record provider \absent{}.
\IF{no successful response exists}
  \STATE Set E0--E6 to \absent{}; record provider status; \RETURN.
\ENDIF
\STATE Extract at most one program; execute under the 30-s cap; ingest the exported LP/MPS exactly.
\STATE Set E0 from construction and ingestion.
\IF{E0 is not \true{}}
  \STATE Set E1--E6 to \absent{}; record the profile; \RETURN.
\ENDIF
\STATE Exhaustively propose and exactly check every map in the declared E1 family.
\IF{no admissible map is found}
  \STATE Set E1 to \unk{} and E2--E6 to \absent{}; record the profile; \RETURN.
\ENDIF
\STATE Set E1 to \true{} and run E2--E6 in order under one 120-s pair-level wall cap.
\STATE Serialize each proposed certificate or witness and independently replay it.
\STATE Emit a decided fact only if replay succeeds; otherwise emit the appropriate typed abstention.
\STATE Store checkpoints, hashes, evidence, baseline records, and the profile.
\end{algorithmic}
\end{algorithm}

Prerequisite normalization follows one rule for every profile: an unmet
prerequisite is \absent{}, whereas N/A is reserved for a relation that is
mathematically inapplicable after all prerequisites hold. This distinction is
used in Tables~\ref{tab:full-results} and~\ref{tab:typed-outcomes}.

\subsection{Deterministic E1 search and applicability}
\label{supp:e1-search}

The evaluation uses no LLM map proposer. The implementation exhaustively runs
its deterministic proposer list, de-duplicates identical maps, and checks each
proposal structurally before it can support any downstream claim. The declared
family evaluated for every E1 profile is
\begin{quote}
\ttfamily\small
stable-name-match; original-name-match; identity;\newline
signature-permutation; binary-complement; affine-lift.
\end{quote}
Name matching yields either a same-space coordinate map or, when the candidate
contains additional variables, a projection map. Identity and signature
permutation require equal dimension. Binary complement is proposed only for
pure binary same-space pairs and is then rechecked exactly. The affine-lift
proposal records projected reference coordinates; E3 separately verifies the
projection/section obligations.

E1 is \true{} if at least one proposal is admissible and \unk{} if the complete
declared family yields none. It is never a certified negative. Once E1 is
\true{}, E2 is applicable only to admissible same-space maps. Seven formal cells
(two GPT-5.4 and five Claude Sonnet 4.6) contain an admissible affine-lift map but
no admissible same-space map; their E2 state is therefore N/A rather than
\unk{}. E3 then records whether the supported affine-lift schema can be
certified.

\begin{algorithm}[t]
\caption{E2 relation under the complete declared same-space map universe}
\label{alg:e2}
\begin{algorithmic}[1]
\REQUIRE exact models $(M_C,M_R)$ and admissible E1 maps $\Phi_{\mathrm{adm}}$
\STATE $\Phi_s\leftarrow\{\varphi\in\Phi_{\mathrm{adm}}:\varphi\text{ is same-space}\}$
\IF{$\Phi_s=\varnothing$}
  \STATE \RETURN N/A.
\ENDIF
\FOR{each $\varphi\in\Phi_s$}
  \STATE Normalize $\varphi(F_C)$ and $F_R$ in a common exact coordinate system.
  \STATE Certify $\varphi(F_C)\subseteq F_R$, or verify a separating point.
  \STATE Certify $F_R\subseteq\varphi(F_C)$, or verify a separating point.
  \STATE Store the per-map relation implied by the two verified directions.
\ENDFOR
\IF{some map is certified equal}
  \STATE \RETURN \true{} with that map and both directional certificates.
\ELSIF{every map has a verified non-equality witness}
  \STATE \RETURN \fwithin{} with the complete map manifest and witnesses.
\ELSIF{the pair-level cap was reached}
  \STATE \RETURN \ur{}.
\ELSE
  \STATE \RETURN \unk{}.
\ENDIF
\end{algorithmic}
\end{algorithm}

\subsection{Independent replay and evidence schema}

\begin{algorithm}[t]
\caption{Independent replay of a proposed relation fact}
\label{alg:replay}
\begin{algorithmic}[1]
\REQUIRE normalized model hashes, named map, relation claim, serialized evidence
\STATE Re-load the exact models and verify all source hashes.
\STATE Re-check map admissibility without consulting constructor booleans.
\STATE Dispatch on evidence type: Farkas multipliers, separating point,
 affine projection/section, objective identity or ranking witness, primal/dual
 optimum certificate, or optimizer bijection.
\STATE Recompute rational identities, domain membership, feasibility, strict
 violations, and right-inverse identities as applicable.
\IF{every obligation holds exactly}
  \STATE \RETURN verified.
\ELSE
  \STATE \RETURN rejected; the proposed decided fact is removed.
\ENDIF
\end{algorithmic}
\end{algorithm}

The replay verifier reads neither natural-language prompts nor solver status
strings. Each evidence record contains the cell and cohort identifiers, exact
model hashes, map identifier, relation and directional obligation, rational
multipliers/points/matrices/vectors, constructor status, verifier status,
checkpoint timing, and source hashes. Reported re-verification denominators
count dimension-verdict artifacts rather than unique cells; one cell can
contribute multiple E4--E6 artifacts.

\subsection{Exactness, support envelope, and timeout semantics}

Gurobi is used for candidate construction/export, LP/MPS parsing support, and
the external value and structural baselines. It is not the E1--E6 certifier.
Exact acceptance uses rational arithmetic in the certifying tradition of
\citet{applegate2007exact,cook2013hybrid}. Continuous linear containment uses
Farkas certificates or exact separating witnesses; bounded pure-binary cases
use complete enumeration where declared. Unsupported general-integer
containment, unimplemented lift schemas, and unavailable optimizer-set negatives
produce typed abstentions rather than guessed booleans.

The 120-s limit is one hard wall-clock cap for the complete E1--E6 pair
evaluator, not a fresh 120 s for every dimension. The evaluator records a
checkpoint before each dimension. If the cap is reached, E0 and all completed
dimensions keep their recorded outcomes; the in-progress dimension and every
unreached successor are set to \ur{}. Consequently, the cap-hit statistic is a
union over cells, whereas the per-dimension resource counts are cumulative.
Provider responses are identified by request fingerprints.

\FloatBarrier
\section{Cohort, Models, Prompt Contract, and Statistical Protocol}
\label{supp:protocol}

\subsection{Cohort construction and eligibility}

\begin{table}[t]
\centering
\small
\begin{tabular}{@{}lrr@{}}
\toprule
Partition & Base problems & Cells/model \\
\midrule
Pilot subset & 20 & 40 \\
Formal primary cohort & 173 & 346 \\
Unsupported reference set & 4 & 8 \\
Malformed/missing & 0 & 0 \\
\midrule
Total & 197 & 394 \\
\bottomrule
\end{tabular}
\caption{Disjoint and exhaustive Bench4Opt paired-base accounting. Only the
173-base formal cohort contributes to primary results.}
\label{tab:cohort}
\end{table}

Bench4Opt contains 394 Structured/Unstructured records, grouped into 197 paired
base problems \citep{wang2025orgeval}. Eligibility is determined from reference
artifacts only. The 20-base pilot subset is defined by sorting canonical base
indices and taking the first 20 reference pairs inside the declared linear
envelope, skipping unsupported base 0016; this gives bases 0000--0015 and
0017--0020. For the remaining 177 bases, inclusion requires (i) both paired
records, (ii) exact E0-successful ingestion of the reference LP, and (iii)
membership in \{continuous LP, bounded pure binary, general-integer MILP\}.
Candidate outputs play no role in selection. The formal stratum counts are 112,
29, and 32, respectively.

\begin{table}[t]
\centering
\small
\setlength{\tabcolsep}{3pt}
\begin{tabular}{@{}ll@{}}
\toprule
Excluded base ID & Ground-truth-only reason \\
\midrule
\texttt{B4O\_BASE\_0016} & quadratic objective \\
\texttt{B4O\_BASE\_0028} & quadratic constraint \\
\texttt{B4O\_BASE\_0182} & quadratic objective \\
\texttt{B4O\_BASE\_0187} & quadratic constraint \\
\bottomrule
\end{tabular}
\caption{Unsupported reference instances. No paired record is malformed or
missing.}
\label{tab:unsupported-ids}
\end{table}

The formal cohort hash is
\begin{center}
\ttfamily\small
 e9ad20181d447143a3b53b377426af88d\
 395753cfd6f733eac0872373376b31c
\end{center}
The accompanying reproducibility archive includes the 197-base eligibility record, the pilot
list, the 173-base manifest, and the four exclusion reasons.

\subsection{Model aliases and generation protocol}

\begin{table}[t]
\centering
\small
\setlength{\tabcolsep}{3pt}
\begin{tabular}{@{}p{0.30\columnwidth}p{0.61\columnwidth}@{}}
\toprule
Field & Setting \\
\midrule
Provider/API & AutoDL; OpenAI-compatible chat completions \\
Model aliases & \texttt{gpt-5.4}; \texttt{claude-sonnet-4-6}; \texttt{Qwen3.5-397B-A17B} \\
Conditions & Structured (artifact key \texttt{full}); Unstructured (\texttt{concise}) \\
Samples & one generation per model/condition cell \\
Temperature / seed & $0.0$ / 7 \\
Maximum output & 4096 tokens \\
Repair/resampling & disabled \\
Output contract & one self-contained program exporting \texttt{candidate.lp} or MPS \\
Request timeout/retries & 120 s; at most two retries \\
Candidate execution & 30-s hard cap \\
Pair evaluator / ORGEval & 120-s hard cap / 120-s hard cap \\
Fallback / substitution & disabled \\
\bottomrule
\end{tabular}
\caption{Generation and evaluation protocol. No model-specific prompt
adaptation is used.}
\label{tab:protocol}
\end{table}

The exact template, model configuration files, request fingerprints, and
credential-free provider-status records are included in the reproducibility archive.
Provider aliases are reported exactly as used by the evaluation client. The
provider did not expose immutable served checkpoint hashes, so model aliases
and request fingerprints are the strongest available identifiers.

\subsection{Software environment}

The local evaluator and the ORGEval-style baseline implementation used in our
harness run under Python 3.9.7 on Windows 11, build 26100, with Gurobi 10.0.2,
NetworkX 2.6.3, and NumPy 1.26.4. Exact verification additionally uses Python's
\texttt{fractions.Fraction}. Provider-side hardware is not observable. CPU and
RAM identifiers were not recorded in the profile metadata; the archive provides
a replay environment specification and a system-information collector.

\subsection{Provider-aware and prerequisite-aware accounting}

Provider delivery status is reported separately from model semantics. A
terminal provider/API failure before a response yields E0--E6 \absent{}. All
346 GPT-5.4 cells have successful responses. Claude Sonnet 4.6 has 337
successful responses and nine terminal read-timeout errors; among successful
responses, 156 are E0 \true{} and 181 are E0 \false{}. Qwen has 329 successful
responses and 17 terminal provider errors; among successful responses, 196 are
E0 \true{}, 131 are E0 \false{}, and two are \uu{}.

Among the 190 Claude Sonnet 4.6 cells without an E0-\true{} candidate, 181
are candidate-level E0 \false{} outcomes and nine are provider absences. All
E1--E6 denominators condition on the 156 E0-\true{} cells.

\subsection{Statistical unit}

Profile counts use the cell as unit; Structured--Unstructured comparisons use
the 173 paired base problems. The paired analysis uses five binary end-to-end
indicators: E0 success, E1 map found, E2 decided, E5 decided, and E6 decided.
Table~\ref{tab:paired-all} reports exact two-sided McNemar tests and 10,000
base-level paired bootstrap resamples with seed 13. Holm correction covers all
15 model-by-indicator tests. These analyses describe the evaluated cohort and
do not establish universal prompt effects.

\FloatBarrier
\section{Complete Formal Results and Baseline Outputs}
\label{supp:results}

\subsection{Stage-wise outcome accounting}

\begin{table}[H]
\centering
\scriptsize
\setlength{\tabcolsep}{1.4pt}
\renewcommand{\arraystretch}{1.02}
\begin{tabular}{@{}p{0.57\columnwidth}rrr@{}}
\toprule
Metric & GPT-5.4 & Sonnet & Qwen \\
\midrule
Cells & 346 & 346 & 346 \\
Successful API responses & 346 & 337 & 329 \\
Provider/API error; E0 \absent{} & 0 & 9 & 17 \\
E0 \true{} & 334 & 156 & 196 \\
Candidate-level E0 \false{} & 12 & 181 & 131 \\
E0 \uu{} & 0 & 0 & 2 \\
E1 map found / E0 \true{} & 277/334 & 130/156 & 164/196 \\
E1 \unk{} & 57 & 26 & 32 \\
E2 decided / E1 \true{} & 114/277 & 45/130 & 87/164 \\
E2 \true{} / \fwithin{} & 108/6 & 41/4 & 84/3 \\
E2 \unk{} / N/A & 116/2 & 56/5 & 50/0 \\
E2 \ur{} & 45 & 24 & 27 \\
E3 decided & 0 & 0 & 0 \\
E4 decided & 139 & 40 & 91 \\
E5 \true{} / \false{} & 107/42 & 36/31 & 79/22 \\
E6 \true{} & 70 & 16 & 58 \\
Evidence artifacts independently verified & 358/358 & 123/123 & 250/250 \\
Cells hitting at least one 120-s cap & 70 & 35 & 39 \\
Worker termination/corrupt profiles & 0 & 0 & 0 \\
\bottomrule
\end{tabular}
\caption{Stage-wise summary after provider and prerequisite normalization.
The model columns are GPT-5.4, Claude Sonnet 4.6, and Qwen3.5-397B-A17B.
Conditional denominators are explicit; complements are not automatically
semantic errors.}
\label{tab:full-results}
\end{table}

% --- begin inlined analysis_support/final_reporting/typed_outcomes_table.tex ---
\begin{table*}[t]
\centering
\scriptsize
\setlength{\tabcolsep}{3.2pt}
\renewcommand{\arraystretch}{1.06}
\begin{tabular}{@{}llrrrrrrrr@{}}
\toprule
Model & Dim. & T & F & F$_{\cal U}$ & U & U$_r$ & U$_u$ & N/A & A \\
\midrule
GPT-5.4 & E0 & 334 & 12 & 0 & 0 & 0 & 0 & 0 & 0 \\
 & E1 & 277 & 0 & 0 & 57 & 0 & 0 & 0 & 12 \\
 & E2 & 108 & 0 & 6 & 116 & 45 & 0 & 2 & 69 \\
 & E3 & 0 & 0 & 0 & 2 & 45 & 0 & 230 & 69 \\
 & E4 & 138 & 1 & 0 & 74 & 64 & 0 & 0 & 69 \\
 & E5 & 107 & 42 & 0 & 59 & 69 & 0 & 0 & 69 \\
 & E6 & 70 & 0 & 0 & 137 & 70 & 0 & 0 & 69 \\
\midrule
Claude Sonnet 4.6 & E0 & 156 & 181 & 0 & 0 & 0 & 0 & 0 & 9 \\
 & E1 & 130 & 0 & 0 & 26 & 0 & 0 & 0 & 190 \\
 & E2 & 41 & 0 & 4 & 56 & 24 & 0 & 5 & 216 \\
 & E3 & 0 & 0 & 0 & 5 & 24 & 0 & 101 & 216 \\
 & E4 & 40 & 0 & 0 & 57 & 33 & 0 & 0 & 216 \\
 & E5 & 36 & 31 & 0 & 28 & 35 & 0 & 0 & 216 \\
 & E6 & 16 & 0 & 0 & 79 & 35 & 0 & 0 & 216 \\
\midrule
Qwen3.5-397B-A17B & E0 & 196 & 131 & 0 & 0 & 0 & 2 & 0 & 17 \\
 & E1 & 164 & 0 & 0 & 32 & 0 & 0 & 0 & 150 \\
 & E2 & 84 & 0 & 3 & 50 & 27 & 0 & 0 & 182 \\
 & E3 & 0 & 0 & 0 & 0 & 27 & 0 & 137 & 182 \\
 & E4 & 91 & 0 & 0 & 40 & 33 & 0 & 0 & 182 \\
 & E5 & 79 & 22 & 0 & 25 & 38 & 0 & 0 & 182 \\
 & E6 & 58 & 0 & 0 & 67 & 39 & 0 & 0 & 182 \\
\bottomrule
\end{tabular}
\caption{Complete prerequisite-normalized typed outcome matrix over all 346 cells per model. $F_{\cal U}$ denotes \textsc{false\_within\_declared\_universe}; $U_r$, $U_u$, and A denote \textsc{unknown\_resource}, \textsc{unknown\_unsupported}, and \textsc{absent}. Every row sums to 346.}
\label{tab:typed-outcomes}
\end{table*}
% --- end inlined analysis_support/final_reporting/typed_outcomes_table.tex ---

\subsection{Complete typed outcome matrix}

Table~\ref{tab:typed-outcomes} is exhaustive: every dimension row totals 346.
For Qwen, E4--E6 are \absent{} in 182 cells because E0 or E1 is not established.
The corresponding E4/E5/E6 \unk{} counts are 40, 25, and 67, and the \ur{}
counts are 33, 38, and 39.

\subsection{External baseline outputs}

\begin{table}[H]
\centering
\scriptsize
\setlength{\tabcolsep}{2.2pt}
\begin{tabular}{@{}p{0.57\columnwidth}rrr@{}}
\toprule
Baseline outcome & GPT-5.4 & Sonnet 4.6 & Qwen3.5 \\
\midrule
Value match: equal & 169 & 54 & 108 \\
Value match: not equal & 76 & 62 & 51 \\
Value match: abstain & 89 & 40 & 37 \\
Value match: absent (no E0-true candidate) & 12 & 190 & 150 \\
ORGEval: equivalent & 168 & 84 & 110 \\
ORGEval: not equivalent & 163 & 71 & 84 \\
ORGEval: unknown & 3 & 1 & 2 \\
ORGEval: absent (no E0-true candidate) & 12 & 190 & 150 \\
\midrule
Execution-success overestimation & 49 & 35 & 25 \\
ORGEval rejection despite E2 equality & 25 & 8 & 18 \\
Strict value-match false acceptance & 0 & 0 & 0 \\
\bottomrule
\end{tabular}
\caption{Baseline outputs and relation-specific certified discrepancies.
Baseline ``absent'' combines every cell without an E0-true candidate, including
provider absence, because no baseline pair exists.}
\label{tab:baseline-full}
\end{table}

An ORGEval rejection despite E2 equality means that the harnessed ORGEval
implementation returns \texttt{not\_equivalent} while E2 certifies mapped
feasible-set equality under a verified map. This is an E2-specific discrepancy,
not a claim of complete profile equality. A strict value-match false acceptance
requires value equality together with an E2-certified strict/incomparable
feasible-set relation; none is observed.

\FloatBarrier
\section{Paired Conditions, Case Study, and Failure Taxonomy}
\label{supp:additional-analysis}

\subsection{Structured versus Unstructured paired statistics}

% --- begin inlined analysis_support/final_reporting/paired_all_models_table.tex ---
\begin{table*}[t]
\centering
\scriptsize
\setlength{\tabcolsep}{3.4pt}
\begin{tabular}{@{}llrrrrrrr@{}}
\toprule
Model & Indicator & S & U & S-only & U-only & $\Delta$ pp [95\% CI] & $p$ & Holm $p$ \\
\midrule
GPT-5.4 & E0 success & 165 & 169 & 3 & 7 & -2.3 [-5.8, +1.2] & 0.344 & 1.000 \\
 & E1 map found & 145 & 132 & 22 & 9 & +7.5 [+1.2, +13.9] & 0.029 & 0.412 \\
 & E2 decided & 62 & 52 & 17 & 7 & +5.8 [+0.6, +11.6] & 0.064 & 0.767 \\
 & E5 decided & 77 & 72 & 12 & 7 & +2.9 [-1.7, +8.1] & 0.359 & 1.000 \\
 & E6 decided & 38 & 32 & 17 & 11 & +3.5 [-2.3, +9.2] & 0.345 & 1.000 \\
\midrule
Claude Sonnet 4.6 & E0 success & 77 & 79 & 38 & 40 & -1.2 [-11.0, +9.2] & 0.910 & 1.000 \\
 & E1 map found & 68 & 62 & 39 & 33 & +3.5 [-6.4, +13.3] & 0.556 & 1.000 \\
 & E2 decided & 29 & 16 & 21 & 8 & +7.5 [+1.7, +13.3] & 0.024 & 0.362 \\
 & E5 decided & 38 & 29 & 21 & 12 & +5.2 [-1.2, +11.6] & 0.163 & 1.000 \\
 & E6 decided & 11 & 5 & 8 & 2 & +3.5 [+0.0, +6.9] & 0.109 & 1.000 \\
\midrule
Qwen3.5-397B-A17B & E0 success & 103 & 93 & 34 & 24 & +5.8 [-2.9, +14.5] & 0.237 & 1.000 \\
 & E1 map found & 89 & 75 & 29 & 15 & +8.1 [+0.6, +15.6] & 0.049 & 0.634 \\
 & E2 decided & 47 & 40 & 17 & 10 & +4.0 [-1.7, +9.8] & 0.248 & 1.000 \\
 & E5 decided & 54 & 47 & 16 & 9 & +4.0 [-1.7, +9.8] & 0.230 & 1.000 \\
 & E6 decided & 34 & 24 & 17 & 7 & +5.8 [+0.6, +11.6] & 0.064 & 0.767 \\
\bottomrule
\end{tabular}
\caption{Base-paired Structured (S) versus Unstructured (U) descriptive analysis for all three models. Exact two-sided McNemar tests use discordant pairs; confidence intervals use 10,000 paired base-level bootstrap resamples with seed 13; Holm correction covers the 15 model-by-indicator tests.}
\label{tab:paired-all}
\end{table*}
% --- end inlined analysis_support/final_reporting/paired_all_models_table.tex ---

No row remains significant after the 15-test Holm correction. The smallest raw
values are Claude Sonnet 4.6 E2 ($p=0.0241$), GPT-5.4 E1 ($p=0.0294$), and Qwen
E1 ($p=0.0488$); their adjusted values are 0.362, 0.412, and 0.634. The table
therefore supports only descriptive coverage patterns.

\subsection{Replayable certified case}

For \texttt{B4O\_BASE\_0031} under the Unstructured condition, the Claude
Sonnet 4.6 candidate and reference each contain five binary variables and 18
inequalities. E1 verifies the identity map. Complete enumeration over all
$2^5=32$ assignments independently verifies both E2 containments. The profile
stores E1/E2 replay status in the map-level containment evidence, where both
directional checks are independently verified. E4, E5, and E6 also
verify positive; E3 is N/A because no auxiliary variables exist.
ORGEval nevertheless returns \texttt{not\_equivalent}, and the value baseline
returns \texttt{not\_equal}, illustrating why baseline labels must be tied to
their precise relation. This is the case visualized in Main Figure~2.

\subsection{Failure taxonomy}

Analysis keeps the following categories disjoint: terminal provider/API error;
successful response without an extractable program; execution/export failure;
unsupported serialized model; no admissible E1 map in the complete declared
family; certified semantic negative; unresolved semantic search; resource cap;
mathematically inapplicable relation; and downstream absence after a failed
prerequisite. This is why E0 failure, E1 coverage, semantic disagreement, and
resource abstention are not collapsed into one error rate.

\FloatBarrier
\section{Reproduction Procedure and Artifact Contract}
\label{supp:reproduction}

\subsection{Reproducibility archive contents}

The reproducibility code/data archive is organized as follows:
\begin{itemize}
\item \texttt{code/certequivbench/}: ingestion, mapping, relation construction,
independent verification, baselines, cohort eligibility, and reporting code;
\item \texttt{artifacts/cohort/}: the 197-base eligibility record, 20 pilot IDs,
173-base manifest, and four unsupported records;
\item \texttt{artifacts/profiles/}: all 1,038 per-cell profiles;
\item \texttt{artifacts/candidates/}: candidate programs where available;
normalized candidate models and their hashes are retained in the per-cell
profiles; exported LP/MPS files were not retained;
\item \texttt{artifacts/baselines/}: per-cell value and ORGEval records;
\item \texttt{artifacts/provider\_status/}: credential-free call status,
model alias, fingerprint, retry, error category, and usage metadata;
\item \texttt{analysis/}: normalization, table-generation, paired-analysis, and
selected-case scripts plus their deterministic outputs;
\item \texttt{environment/}: replay requirements, the recorded provenance, and
a system-information collector.
\end{itemize}
Raw API keys, author identities, absolute personal paths, and provider response
text are excluded. The third-party Bench4Opt corpus is not duplicated; the
manifest supplies stable IDs and hashes, and the README explains how to place an
authorized local copy for full re-ingestion.

\subsection{Replay sequence and invariants}

A reproducer first verifies the archive manifest and cohort hash, installs the
specified environment, and places Bench4Opt at the documented local path. The
replay then (i) executes the reference-support eligibility check, (ii) ingests
the recorded candidate models without provider access, (iii) independently
verifies all certificates and witnesses, (iv) rebuilds normalized profile and
baseline summaries, (v) recomputes the 10,000-resample paired analysis with seed
13, and (vi) compiles the appendix and checklist. The process fails closed on a cohort
hash, cell identity, model hash, evidence verdict, conditional denominator,
provider-category partition, or cap-accounting mismatch.

\subsection{Claim-to-evidence index}

\begin{itemize}
\setlength{\itemsep}{1pt}
\setlength{\parsep}{0pt}
\setlength{\topsep}{2pt}
\item \textbf{Profile is not a hierarchy.} Main Secs.~1, 3.1, and 3.4; Supp.~A and E. Replay the non-implication examples and prerequisite DAG.
\item \textbf{E2 equality is sound.} Main Sec.~3.3; Supp.~C. Check two directional Farkas certificates or complete finite enumeration.
\item \textbf{E2 aggregate negatives are scoped.} Main Sec.~3.3; Supp.~C.3. Check the complete declared map manifest and a verified witness for every map.
\item \textbf{E3 affine-lift equality is sound.} Main Sec.~3.3; Supp.~D. Check projection/section matrices, two containments, and the exact right-inverse identity.
\item \textbf{E4--E6 are distinct.} Main Sec.~3.4; Supp.~E. Check objective identities or ranking witnesses, optimum certificates, and optimizer bijections.
\item \textbf{Decided facts are independently checked.} Main Secs.~4.1--4.2; Supp.~F. Replay serialized evidence with the independent verifier.
\item \textbf{Three-model counts share one cohort.} Main Secs.~5.1--5.2; Supp.~G--H. Check the cohort manifest/hash and per-cell profile matrix.
\item \textbf{Baseline discrepancies are relation-specific.} Main Sec.~5.3; Supp.~H--I. Join each baseline record to its E2/E4--E6 profile by cell ID.
\item \textbf{Provider failures are not E0 failures.} Main Secs.~5.1--5.2; Supp.~G.4 and H. Check provider-status records and the normalized typed-outcome matrix.
\item \textbf{Paired prompt results are descriptive.} Main Sec.~5.4; Supp.~G.5 and I.1. Recompute the paired matrix, seed-13 bootstrap, McNemar tests, and Holm correction.
\end{itemize}

\section{Interpretive Clarifications}
\label{supp:clarifications}

\paragraph{Is \system{} a new dataset?}
No. It is an evaluator and certifying artifact system; Bench4Opt supplies the
experimental instances.

\paragraph{Do E0--E6 define an equivalence ladder?}
No. E0 and E1 are prerequisites/alignments; E2 and E3 compare different spaces;
E4--E6 ask distinct objective and solution questions.

\paragraph{Does E1 \unk{} mean non-equivalence?}
No. It means the complete finite declared family yielded no admissible map.

\paragraph{When is a state N/A rather than \absent{}?}
N/A requires satisfied prerequisites and mathematical inapplicability. For
example, E2 is N/A when E1 found only an affine-lift map, and E3 is N/A when a
same-space map leaves no auxiliary variables. If E0 or E1 fails, downstream
relations are \absent{}.

\paragraph{Why is E3 decided on zero formal cells?}
No formal cell both required and satisfied the one implemented affine-lift
certificate schema. This is a coverage result, not a counterexample to the E3
theorem.

\paragraph{Why can E5 be positive while a value baseline says not equal?}
The baseline may compare a solver-derived or relaxed scalar, whereas E5 is
accepted only from exact relation-specific evidence for the declared domains.

\paragraph{Why can ORGEval reject an E2-equal pair?}
A structural graph comparison can distinguish alternative encodings even when
exact mapped feasible sets coincide. The discrepancy is relation-specific.

\paragraph{Why only one generation per cell?}
The protocol evaluates one output per cell and excludes repair or best-of-$k$
selection. It therefore does not estimate within-model sampling variance.

\section{Limitations, Machine-Readable Contract, and Responsible Use}
\label{supp:limitations}

The guarantees are conditional on exact ingestion, the named admissible map,
the supported certificate envelope, and the correctness of the independent
verifier. They do not prove that the reference model captures the source
natural-language intent. The mapping family is finite, E3 implements one lift
schema, general-integer containment remains incomplete, and E6 has no certified
negative type. Exact procedures may exhaust the pair-level resource cap. The
empirical study uses three provider aliases, one fixed prompt, and one sample
per condition; it is diagnostic rather than a universal ranking.

\subsection{Machine-readable profile contract}

A portable profile record exposes at least
\begin{quote}
\ttfamily\scriptsize
cell\_id, base\_id, model\_alias, condition, cohort\_hash;\newline
provider\_status, request\_fingerprint, candidate\_status, e0\_status;\newline
e1\_status, declared\_map\_family, searched\_maps[], admissible\_maps[];\newline
e2..e6: status, relation, map\_id, reason, evidence\_id, verifier\_status;\newline
elapsed\_seconds, pair\_cap, checkpoint, baselines, source\_hashes,
environment\_fingerprint.
\end{quote}
Provider status is separate from E0. Each downstream entry carries its map and
prerequisite context. Replay status is recorded either at the relation level or
in the referenced map-level evidence object; the verifier resolves the
referenced evidence before assigning the relation outcome.

\subsection{Threats to validity}

\paragraph{Construct validity.}
The reference is treated as ground truth. \system{} measures candidate-to-reference
semantics, not whether either model faithfully captures the source problem.

\paragraph{Internal validity.}
Separate construction and replay paths, exact arithmetic, hashes, and
fail-closed routing reduce but cannot eliminate implementation defects. The
provider-aware normalization is deterministic and preserves the source
records.

\paragraph{External validity.}
Conclusions are limited to the supported Bench4Opt cohort, the three served
aliases, the fixed prompt/output contract, and the recorded software envelope.

\paragraph{Statistical conclusion validity.}
Primary tables are exact finite-cohort counts. Paired tests characterize the
173 bases and are not estimates of an open-ended model population.

\subsection{Responsible use}

The evaluator may expose generated code and optimization artifacts. Any shared
archive must remove credentials, identities, absolute personal paths, and
non-redistributable provider content. Certification against a reference model
must not be presented as validation of a real-world deployment without separate
checks of reference correctness, data provenance, operational constraints, and
domain consequences.

\end{document}